\documentclass{article}
\pdfoutput=1
\usepackage{microtype}
\usepackage{bm}
\usepackage{graphicx}
\usepackage{subfigure}
\usepackage{amsmath}
\usepackage{amsfonts}
\usepackage{amsthm}
\usepackage{booktabs} % for professional tables
\newtheorem{prop}{Proposition}
\usepackage{hyperref}

\usepackage[arxiv]{icml2018}
\usepackage[draft]{todonotes}
\icmltitlerunning{Transductive Adversarial Networks (TAN)}

\begin{document}

\twocolumn[
\icmltitle{Transductive Adversarial Networks (TAN)}



\icmlsetsymbol{equal}{*}

\begin{icmlauthorlist}
\icmlauthor{Sean Rowan}{ucl}
\end{icmlauthorlist}

\icmlaffiliation{ucl}{University College London}

\icmlcorrespondingauthor{}{sean.rowan.16@ucl.ac.uk}

\icmlkeywords{Machine Learning, domain-adaptation, Transfer Learning, Generative Adversarial Networks, Semi-supervised}

\vskip 0.3in
]



\printAffiliationsAndNotice{}  % leave blank if no need to mention equal contribution

\begin{abstract}
Transductive Adversarial Networks (TAN) is a novel domain-adaptation machine learning framework that is designed for learning a conditional probability distribution on unlabelled input data in a target domain, while also only having access to: (1) easily obtained labelled data from a related source domain, which may have a different conditional probability distribution than the target domain, and (2) a marginalised prior distribution on the labels for the target domain. TAN leverages a fully adversarial training procedure and a unique generator/encoder architecture which approximates the transductive combination of the available source- and target-domain data. A benefit of TAN is that it allows the distance between the source- and target-domain label-vector marginal probability distributions to be greater than 0 (i.e. different tasks across the source and target domains) whereas other domain-adaptation algorithms require this distance to equal 0 (i.e. a single task across the source and target domains). TAN can, however, still handle the latter case and is a more generalised approach to this case. Another benefit of TAN is that due to being a fully adversarial algorithm, it has the potential to accurately approximate highly complex distributions. Theoretical analysis demonstrates the viability of the TAN framework.
\end{abstract}

\section{Introduction}

The scenario of having access to a small amount of labeled data but a large amount of unlabelled data is a common one in practice. In an idealised learning situation, the conditional probability distribution between the input vector and the label vector across the sets of labeled and unlabelled data are equal. However, typically this does not occur in practice. Instead, the small amount of labeled data that is accessible is usually either significantly simpler than the encountered unlabelled data, or comes from a different domain with a different conditional probability distribution between its input vector and label vector. These two practical cases can be considered the same from a learning point-of-view as the latter practical case \cite{pan2010survey}. 

In the standard domain-adaptation learning scenario, it is expected that the labelled and unlabelled input vectors can be drawn from unique marginal probability distributions. However, it is required that the label-vector marginal probability distribution for the labelled and unlabelled sets of data are equal and match the available labelled set of data \cite{pan2010survey}. This is demonstrated in the following example involving the MNIST (hand-drawn digits) and SVHN (house numbers from Google StreetView images) datasets. 

\begin{figure}[ht]
\vskip 0.2in
\begin{center}
\centerline{\includegraphics[width=0.6\columnwidth]{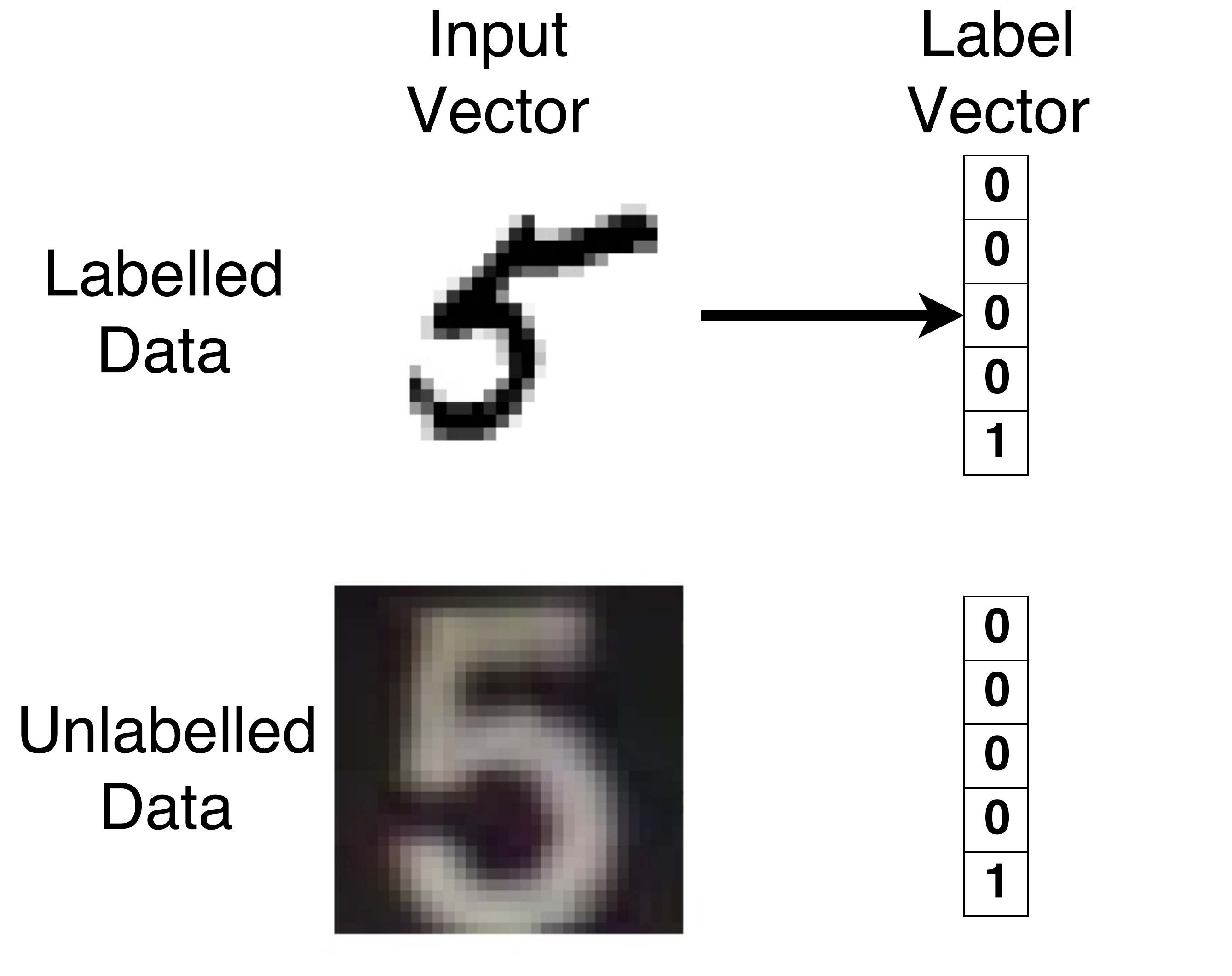}}
\caption{The standard domain-adaptation learning scenario where the label-vector marginal probability distributions across domains are expected to be equal. In this example learning scenario, the labelled data are pairs of both an image of a hand-drawn number from 1 through 5 and a 5-dimensional one-hot encoded vector that encodes the numerical representation of the input image. The unlabelled data are images of house numbers from 1 through 5. There are common features in the input vectors across domains that allow a domain-adaptation learning algorithm to assign labels to the unlabelled input vectors using the available labelled and unlabelled data.}
\label{icml-historical}
\end{center}
\vskip -0.2in
\end{figure}
Now consider a generalised domain-adaptation learning scenario where both the input-vector and the label-vector marginal probability distributions across domains are not expected to be equal. This generalised scenario motivates the design of TAN. The scenario is demonstrated in the following example involving the MNIST and SVHN datasets.

\begin{figure}[ht]
\vskip 0.2in
\begin{center}
\centerline{\includegraphics[width=0.6\columnwidth]{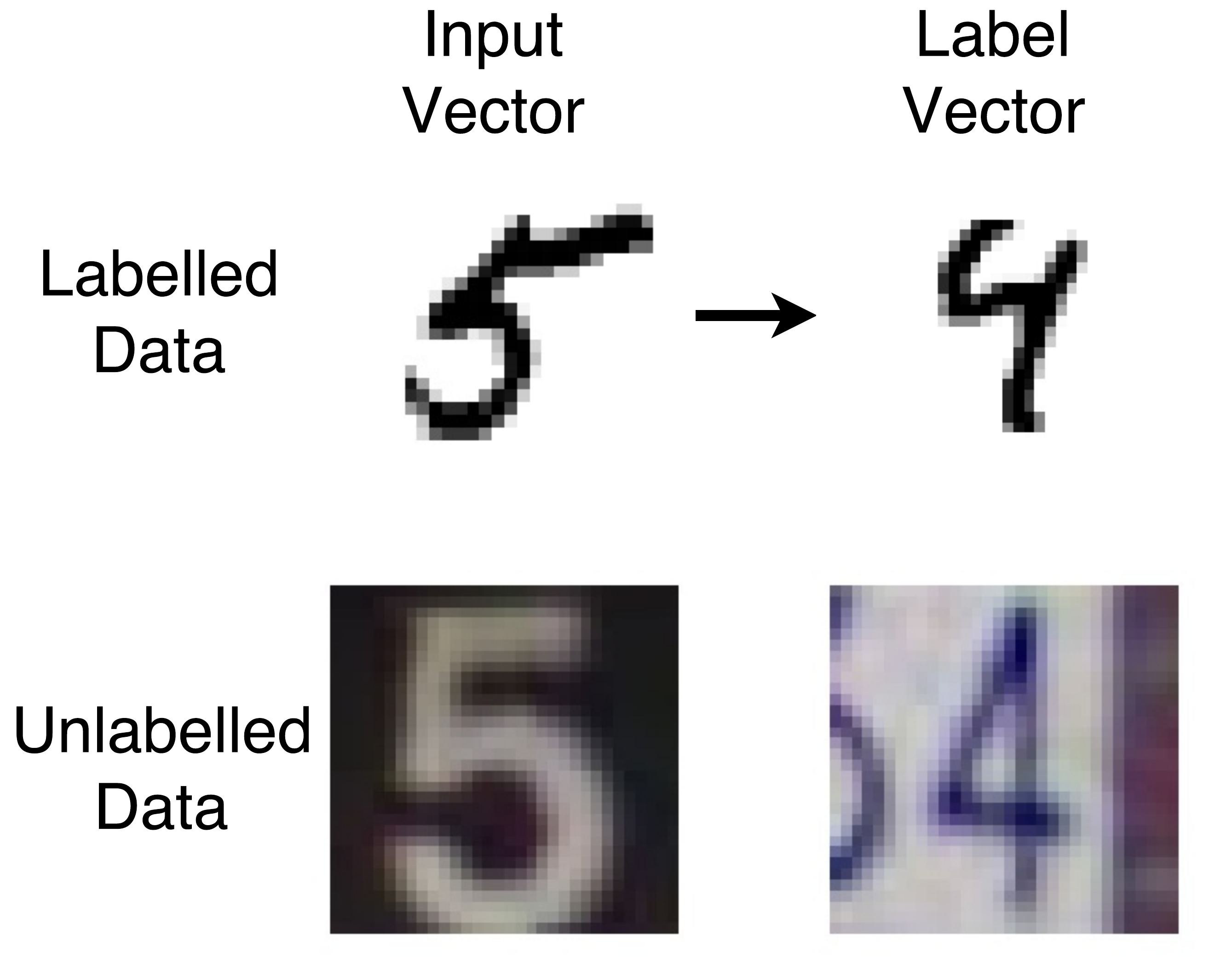}}
\caption{A generalised domain-adaptation learning scenario where both the input-vector and the label-vector marginal probability distributions across domains are not expected to be equal. In this example learning scenario, the labelled data are pairs of a hand-drawn single-digit image of an odd number and a hand-drawn single-digit image of the previous even number. The unlabelled data pairs are the same except they are images of house numbers.}
\label{icml-historical}
\end{center}
\vskip -0.2in
\end{figure}

This generalised domain-adaptation learning scenario is distinct from the style-transfer learning problem. In style transfer, learning occurs only on a single marginalised input vector across domains, and does not involve a corresponding conditional label vector, thus significantly reducing the scope of applications \cite{jing2017neural}.

We now further motivate the usefulness of an algorithm that can learn a conditional probability distribution within the generalised domain-adaptation learning scenario with a real-world application. Consider the problem of human drug discovery. In a human drug discovery scenario, there is no data available about how an experimental drug molecule might bind to known human protein structures due to the difficulty in testing new drugs on live human subjects. However, there is ample data available on how an experimental drug molecule can bind to known yeast cell protein structures due to the free ability to test new drugs on these cells. In this scenario, the yeast cell experiments represent the source domain and the human experiments represent the target domain. The yeast cell protein structure is the input vector of the source domain and the experimental drug for yeast cells is the label vector of the source domain. Likewise, the human protein structure is the input vector of the target domain and the experimental drug for humans is the unknown label vector of the target domain. The learning goal is to generate a shortlist of potential candidate drugs for further human testing. Such an algorithm would be highly valuable in discovering new drugs that are suitable for humans with fewer human drug trials.

%However, if the labelled and unlabelled sets of data are not mutually exclusive in their conditional probability distributions between their inputs and labels (with the unlabelled domain leveraging a prior distribution on its labels), then it is possible to transductively learn an optimal conditional probability distribution between the inputs and labels across the labelled and unlabelled sets of data. This means that with a suitable prior, a transductive combination of labelled and unlabelled data can be approximated for arbitrary combinations of domain-specific conditional probability distributions. This scenario is what motivated the creation of the TAN framework. 

%In this more generalised scenario than the previous scenario, it is only necessary that the conditional probability distributions between the input vector and the label vector across domains are not mutually exclusive for the possibility of transductively learning a conditional probability distribution to occur [PROOF]. This means that the labelled and unlabelled data can be highly dissimilar, yet not entirely dissimilar, on both the input vector and the label vector for successful domain-adaptation learning to potentially occur.

\section{Related Work}

A good overview of transfer learning research and terminology can be found at: \cite{pan2010survey}, we follow this terminology.

There are two prior works that form the basis for the TAN framework: 1) Generative Adversarial Networks (GAN) \cite{goodfellow2014generative} and 2) Adversarially Learned Inference (ALI) \cite{dumoulin2016adversarially} (which is equivalent to BiGAN \cite{donahue2016adversarial}). TAN leverages the general theoretical results from the GAN framework (the ALI framework leverages the GAN results as well) but utilises the ALI framework's training procedure as a component of the unique TAN algorithm.

In the GAN framework, a two-player zero-sum game between adversarial learning agents is played where one agent (the generator) learns to generate convincing fake data, while the other agent (the discriminator) learns to discern generated data from sampled data which comes from an unknown distribution. The generator learns a transfer function that converts an inputted Gaussian-noise vector into convincing data that matches the unknown data distribution on convergence of the adversarial game.

The ALI framework (and also the BiGAN framework) extends the GAN framework by simultaneously learning a reverse transfer function that maps inputted data back to the Gaussian-noise vector which generated it, allowing the ability to finely control the features of the generated data with interpolations in the Gaussian-noise vector. The ALI framework by itself does not allow the ability to learn conditional probability distributions on inputted conditional data pairs \cite{arora2017theoretical}. 

The GAN framework can directly learn conditional probability distributions, and has also previously been formulated for the standard domain-adaptation learning scenario \cite{tzeng2017adversarial}. However, the GAN framework and its variants are not suited to the generalised domain-adaptation learning scenario because the discriminator requires label-vectors\footnote{`Label-vector' here means the actual data that the discriminator discerns as being real or fake, and not the real/fake label for the data inputted to the discriminator. Also, the input-vector for the domain-adaptation problem is inputted to the generator along with the Gaussian noise vector.} that come from the same marginalised probability distribution across the real and fake (i.e. source and target in this case) domains. TAN allows the label-vectors to come from different marginalised probability distributions across the real and fake domains.

$\Delta$-GAN \cite{gan2017triangle} is structurally similar to TAN in that it also leverages the GAN theoretical results and the ALI training procedure, however it is built for the more restrictive inductive transfer learning task and cannot handle the transductive transfer learning task. This means that $\Delta$-GAN requires paired input/label training data in both the source and target domain, whereas TAN only requires paired input/label training data in the source domain, unlabelled input data in the target domain and a marginalised prior distribution on the label-vector distribution in the target domain.

\section{TAN Framework}
In this section, we first outline the TAN probabilistic model, then we define the training procedure of this probabilistic model and finally we conclude with a proof that establishes the global optimality and convergence properties of the TAN framework. 

TAN leverages both a GAN network and an ALI network, but with a shared generator across the two networks. The GAN network is trained normally and exclusively on the source-domain data. The ALI network is also trained normally but exclusively on the target-domain unlabelled input data and a prior on the target-domain label data. A unique training procedure which combines the two networks via the generator forces the shared generator and the ALI network's encoder to accommodate the statistics of both the source and target domain, and leads to convergence at a unique global optimum under mild assumptions (the exact same convexity assumptions from the original GAN framework formulation \cite{goodfellow2014generative}). The trained encoder can then be used as an inference model on the target-domain unlabelled input-data.
\subsection{Model}
We first define our terms as follows. $\bm{x}$ is the input-vector and $\bm{z}$ is the label-vector. $p_s(\bm{x},\bm{z})$ is the joint data distribution from the source domain. $p_s(\bm{z})$ is the marginalised label-vector data distribution from the source domain. $p_t(\bm{x})$ is the unlabelled input-vector data distribution from the target domain. $\widetilde{p}_t(\bm{z})$ is the label-vector prior distribution from the target domain. $G_x(\bm{x}|\bm{z}; \theta_{gx})$ is the shared generator function. $G_z(\bm{z}|\bm{x}; \theta_{gz})$ is the encoder function. $y$ is the binary classification label of the inputted data to a discriminator. $y=1$ for samples from the distribution that the discriminator learns to support. $D_s(y|\bm{x},\bm{z}; \theta_{ds})$ is the source-domain discriminator function. $D_t(y|\bm{x},\bm{z}; \theta_{dt})$ is the target-domain discriminator function.

The source-domain value function is:
\begin{equation}\label{eq:source-value}
\begin{split}
\min_{G} \max_{D} V_s(G, D) = \mathbb{E}_{(\bm{x},\bm{z})\sim p_s(\bm{x},\bm{z})}[\log D_s(\bm{x}, \bm{z})]\\
+ \mathbb{E}_{\bm{z}\sim p_{s}(\bm{z})}[\log (1 - D_s(G_x(\bm{z}), \bm{z}) )].
\end{split}
\end{equation}
The target-domain value function is:
\begin{equation}\label{eq:target-value}
\begin{split}
\min_{G} \max_{D} V_t(G, D) = \mathbb{E}_{\bm{x}\sim p_{t}(\bm{x})}[\log D_t(\bm{x}, G_z(\bm{x}))]\\
+  \mathbb{E}_{\bm{z}\sim \widetilde{p}_{t}(\bm{z})}[\log (1 - D_t(G_x(\bm{z}), \bm{z}))].
\end{split}
\end{equation}
In practice, the value functions are reworked such that the generator maximises an inverted expression whose gradient is stronger when the discriminator's output saturates, as in the original GAN paper \cite{goodfellow2014generative}. Also in practice, the logarithmic functions are replaced with the Wasserstein distance metric which prevents saturation and provides better experimental performance \cite{arjovsky2017wasserstein}. We start with the above original expressions for the value functions in order to make the following TAN theoretical results a more straightforward extension of the original GAN results.

\subsection{Training Procedure}
The TAN training procedure is as follows. For $m$ steps, the source-domain value function (eq. \ref{eq:source-value}) is iteratively solved using stochastic gradient descent. Then for $n$ steps, the target-domain value function (eq. \ref{eq:target-value}) is iteratively solved using stochastic gradient descent. The two value functions share a common generator function, $G_x(\bm{x}|\bm{z}; \theta_{gx})$. This shared generator function learns to accommodate both the source and target domain data, which allows global optimality in the entire TAN framework, as shown in the next section.
\begin{algorithm}[h!]
   \caption{The TAN Training Procedure}
   \label{alg:example}
\begin{algorithmic}
   \STATE $\theta_{gx}$, $\theta_{gz}$, $\theta_{ds}$, $\theta_{dt}$ $\leftarrow$ initialise network parameters
   \REPEAT
	\FOR{$m$ steps} 
	\FOR{$k$ steps}
   		\STATE $(\bm{x},\bm{z})^{(1)}, \dots, (\bm{x},\bm{z})^{(M)} \sim p_s(\bm{x},\bm{z})$ 
		\STATE $\bm{z}^{(1)}, \dots, \bm{z}^{(M)} \sim p_s(\bm{z})$
		\STATE $\hat{\bm{x}}^{(j)} \sim G_x \left(\bm{z}^{(j)}\right), \ \ \ \ \ \ j = 1, \dots, M $ 
		\STATE $\rho_r^{(i)} \leftarrow D_s\left((\bm{x},\bm{z})^{(i)}\right), \ \ \ \ \ \ i = 1, \dots, M $ 
		\STATE $\rho_g^{(j)} \leftarrow D_s\left(\hat{\bm{x}}^{(j)}, \bm{z}^{(j)}\right), \ \ \ \ \ \ j = 1, \dots, M $
		\STATE $\mathcal{L}_d \leftarrow -\frac{1}{M} \left( \sum_{i=1}^{M} \log \left( \rho_r^{(i)} \right) + \sum_{j=1}^{M} \log \left( 1 - \rho_g^{(j)} \right) \right)$
		\STATE $\theta_{ds} \leftarrow \theta_{ds} - \nabla_{\theta_{ds}}\mathcal{L}_d$
	\ENDFOR
	\STATE $\bm{z}^{(1)}, \dots, \bm{z}^{(M)} \sim p_s(\bm{z})$
	\STATE $\hat{\bm{x}}^{(j)} \sim G_x \left(\bm{z}^{(j)}\right), \ \ \ \ \ \ j = 1, \dots, M $ 
	\STATE $\rho_g^{(j)} \leftarrow D_s\left(\hat{\bm{x}}^{(j)}, \bm{z}^{(j)}\right), \ \ \ \ \ \ j = 1, \dots, M $
	\STATE $\mathcal{L}_g \leftarrow \frac{1}{M} \sum_{j=1}^{M} \log \left(1 - \rho_g^{(j)} \right)$
	\STATE $\theta_{gx} \leftarrow \theta_{gx} - \nabla_{\theta_{gx}}\mathcal{L}_g$
	\ENDFOR
	\FOR{$n$ steps}
   	\STATE $\bm{x}^{(1)}, \dots, \bm{x}^{(M)} \sim p_t(\bm{x})$ 
	\STATE $\bm{z}^{(1)}, \dots, \bm{z}^{(M)} \sim \widetilde{p}_t(\bm{z})$
	\STATE $\hat{\bm{z}}^{(i)} \sim G_z \left( \bm{x}^{(i)}\right), \ \ \ \ \ \ i = 1, \dots, M $ 
	\STATE $\hat{\bm{x}}^{(j)} \sim G_x \left( \bm{z}^{(j)}\right), \ \ \ \ \ \ j = 1, \dots, M $
	\STATE $\rho_e^{(i)} \leftarrow D_t(\bm{x}^{(i)}, \hat{\bm{z}}^{(i)}), \ \ \ \ \ \ i = 1, \dots, M $ 
	\STATE $\rho_g^{(j)} \leftarrow D_t(\hat{\bm{x}}^{(j)}, \bm{z}^{(j)}), \ \ \ \ \ \ j = 1, \dots, M $
	\STATE $\mathcal{L}_d \leftarrow -\frac{1}{M} \left( \sum_{i=1}^{M} \log \left( \rho_e^{(i)} \right) + \sum_{j=1}^{M} \log \left( 1 - \rho_g^{(j)} \right) \right)$
	\STATE $\mathcal{L}_g \leftarrow -\frac{1}{M} \left( \sum_{i=1}^{M} \log \left( 1 - \rho_e^{(i)} \right) + \sum_{j=1}^{M} \log \left(\rho_g^{(j)} \right) \right)$
	\STATE $\theta_{dt} \leftarrow \theta_{dt} - \nabla_{\theta_{dt}}\mathcal{L}_d$
	\STATE $\theta_{gx} \leftarrow \theta_{gx} - \nabla_{\theta_{gx}}\mathcal{L}_g$
	\STATE $\theta_{gz} \leftarrow \theta_{gz} - \nabla_{\theta_{gz}}\mathcal{L}_g$
	\ENDFOR
   \UNTIL{convergence}
\end{algorithmic}
\end{algorithm}

\subsection{Global Optimality and Convergence Proof}

\begin{prop}

The global optimum across $\min_{G} \max_{D} V_s(G, D)$ and $\min_{G} \max_{D} V_t(G, D)$ is achieved at: $G_z(\bm{z}|\bm{x}; \theta_{gz}^*) = \frac{p_s(\bm{x},\bm{z})\widetilde{p}_t(\bm{z})}{p_s(\bm{z})p_t(\bm{x})}$.
\end{prop}

\begin{proof}
By straightforward extension of the proof in \cite{goodfellow2014generative}, the following result is achieved on convergence of $\min_{G} \max_{D} V_s(G, D)$. 
\begin{equation}
p_s(\bm{x},\bm{z}) = G_x(\bm{x}|\bm{z};\theta_{gx}^*)p_s(\bm{z}).
\end{equation}
Similarly, by straightforward extension of the proof in \cite{dumoulin2016adversarially}, the following result is achieved on convergence of $\min_{G} \max_{D} V_t(G, D)$.
\begin{equation}
G_x(\bm{x}|\bm{z};\theta_{gx}^*)\widetilde{p}_t(\bm{z}) = G_z(\bm{z}|\bm{x}; \theta_{gz}^*)p_t(\bm{x}).
\end{equation}

The requirement on convergence from \cite{goodfellow2014generative} and \cite{dumoulin2016adversarially} for $\min_{G} \max_{D} V_s(G, D)$ and $\min_{G} \max_{D} V_t(G, D)$ is that $D$ is allowed to reach its optimum at each training step, given $G$. $D_s$ and $D_t$ are each allowed to reach their optimum at each training step of their respective value functions given their respective $G$. Therefore, $\min_{G} \max_{D} V_s(G, D)$ and $\min_{G} \max_{D} V_t(G, D)$ will simultaneously converge if $V_s(G, D)$ and $V_t(G, D)$ are convex in $G$.

Therefore, on simultaneous convergence of the above two value functions,
\begin{equation}
\begin{split}
G_x(\bm{x}|\bm{z};\theta_{gx}^*) &= \frac{G_z(\bm{z}|\bm{x}; \theta_{gz}^*)p_t(\bm{x})}{\widetilde{p}_t(\bm{z})}\\
&= \frac{p_s(\bm{x},\bm{z})}{p_s(\bm{z})}.
\end{split}
\end{equation}
Finally,
\begin{equation}
G_z(\bm{z}|\bm{x}; \theta_{gz}^*) = \frac{p_s(\bm{x},\bm{z})\widetilde{p}_t(\bm{z})}{p_s(\bm{z})p_t(\bm{x})}.
\end{equation}
\end{proof}

\section{Experiments}
We are currently performing extensive experiments and will release the details of these experiments in future versions of this paper.
\section{Conclusion}
We have established Transductive Adversarial Networks (TAN), which learns a conditional probability distribution on unlabelled input data in a target domain while also only having access to: (1) easily obtained labelled data from a related source domain, which may have a different conditional probability distribution than the target domain, and (2) a marginalised prior distribution on the labels for the target domain. Theoretical analysis has demonstrated the viability of the TAN framework, suggesting that the TAN framework could prove useful for further applications.

\section*{Acknowledgements}
Thank you to Sarah Murphy and Lewis Moffat for helpful discussions about the potential applications of TAN.

\bibliographystyle{abbrv}
\bibliography{TAN_2018}

\end{document}